\documentclass[11pt]{article}

\usepackage{dsfont}

\def\colorful{0}

\usepackage{amsmath,amsthm,amsfonts,amssymb,epsfig,color,float,graphicx,verbatim, enumitem}
\usepackage{multirow}
\usepackage{algorithm}
\usepackage[noend]{algpseudocode}
\usepackage{bbm}

\newif\ifhyper\IfFileExists{hyperref.sty}{\hypertrue}{\hyperfalse}
\hypertrue
\ifhyper\usepackage{hyperref}\fi

\usepackage{enumitem}

\usepackage{framed}
\usepackage{nicefrac}

\def\nnewcolor{1}
\ifnum\nnewcolor=1

\fi
\ifnum\nnewcolor=0

\fi

\ifnum\colorful=1

\else

\fi

\newtheorem{theorem}{Theorem}[section]

\newtheorem{assumption}[theorem]{Assumption}

\newtheorem{lemma}[theorem]{Lemma}
\newtheorem{informal theorem}[theorem]{Theorem (informal statement)}

\newtheorem{proposition}[theorem]{Proposition}
\newtheorem{corollary}[theorem]{Corollary}

\newtheorem{fact}[theorem]{Fact}

\theoremstyle{definition}
\newtheorem{definition}[theorem]{Definition}
\newcommand{\eqdef}{\stackrel{{\mathrm {\footnotesize def}}}{=}}

\theoremstyle{problem}
\newtheorem{problem}[theorem]{Problem}

\newcommand{\bx}{\mathbf{x}}
\newcommand{\by}{\mathbf{y}}
\newcommand{\bs}{\mathbf{s}}
\newcommand{\bt}{\mathbf{t}}

\newcommand{\bu}{\mathbf{u}}
\newcommand{\bw}{\mathbf{w}}

\newcommand{\tbx}{\tilde{\bx}}

\newcommand{\lwe}{\mathrm{LWE}}
\newcommand{\D}{\mathcal{D}}

\newcommand{\smp}{\mathrm{sample}}
\newcommand{\scr}{\mathrm{secret}}
\newcommand{\ns}{\mathrm{noise}}
\newcommand{\nsv}{z}

\newcommand{\Dgaus}{D^{\mathcal{N}}}
\newcommand{\Dexp}{D^{\mathrm{expand}}}
\newcommand{\Dcol}{D^{\mathrm{collapse}}}
\renewcommand{\mod}{\mathrm{mod}}
\newcommand{\gaus}{\mathcal{N}}

\newcommand{\negl}{\mathrm{negl}}

\newcommand{\LTF}{\mathrm{LTF}}

\newcommand{\idt}{\mathbf{1}}

\newcommand{\bv}{\mathbf{v}}

\newcommand{\bX}{\bx}

\newcommand{\R}{\mathbb{R}}
\newcommand{\Z}{\mathbb{Z}}
\newcommand{\N}{\mathbb{N}}
\newcommand{\E}{\mathbf{E}}

\newcommand{\eps}{\epsilon}

\newcommand{\pr}{\mathbf{Pr}}

\newcommand{\poly}{\mathrm{poly}}

\newcommand{\sgn}{\mathrm{sign}}
\newcommand{\sign}{\mathrm{sign}}

\newcommand{\opt}{\mathrm{OPT}}

\newcommand{\relu}{\mathrm{ReLU}}

\title{Near-Optimal Cryptographic Hardness of Agnostically Learning Halfspaces 
and ReLU Regression under Gaussian Marginals}

\author{
	Ilias Diakonikolas\thanks{Supported by NSF Medium Award CCF-2107079,
NSF Award CCF-1652862 (CAREER), a Sloan Research Fellowship, and
a DARPA Learning with Less Labels (LwLL) grant.}\\
	UW Madison\\
	{\tt ilias@cs.wisc.edu}\\
	\and
	Daniel M. Kane\thanks{Supported by NSF Medium Award CCF-2107547,
NSF Award CCF-1553288 (CAREER), a Sloan Research Fellowship, and a grant from CasperLabs.}\\
	UC San-Diego \\
	{\tt dakane@ucsd.edu}\\
	\and
	Lisheng Ren\thanks{Supported by NSF Award CCF-1652862 (CAREER) 
	and a DARPA Learning with Less Labels (LwLL) grant.}\\
	UW Madison\\
	{\tt lren29@wisc.edu}\\
}

\begin{document}

\maketitle

\setcounter{page}{0}

\thispagestyle{empty}

\begin{abstract}
We study the task of agnostically learning halfspaces under the Gaussian distribution.
Specifically, given labeled examples $(\bx,y)$ from an unknown distribution on $\R^n \times \{ \pm 1\}$, 
whose marginal distribution on $\bx$ is the standard Gaussian and the labels $y$ can be arbitrary, 
the goal is to output a hypothesis with 0-1 loss $\opt+\eps$, where $\opt$ is the 0-1 loss of the best-fitting halfspace.
We prove a near-optimal computational hardness result for this task, under the  widely believed 
sub-exponential time hardness of the Learning with Errors (LWE) problem. Prior hardness results are either
qualitatively suboptimal or apply to restricted families of algorithms. Our techniques extend to 
yield near-optimal lower bounds for related problems, including ReLU regression.
\end{abstract}

\newpage

\section{Introduction}

A halfspace or Linear Threshold Function (LTF) is any Boolean-valued 
function $f: \R^n \to \{ \pm 1\}$ of the form
$f(\bx) = \sgn \left(\langle \bw, \bx \rangle - t \right)$,
where $\bw \in \R^n$ is called the weight vector and $t \in \R$ is called the threshold.
Here the univariate function $\sign: \R \to \{ \pm 1\}$ is defined as $\sgn(u)=1$ 
if $u \geq 0$ and $\sgn(u)=-1$ otherwise.
The task of learning an unknown halfspace 
is a classical problem in machine learning that has been extensively studied since the 1950s,
starting with the Perceptron algorithm~\cite{Rosenblatt:58}, and
has lead to practically important techniques such as SVMs~\cite{Vapnik:98}
and AdaBoost~\cite{FreundSchapire:97}. In the realizable setting~\cite{Valiant:84},
halfspaces are known to be efficiently learnable (see, e.g.,~\cite{MT:94}) 
without distributional assumptions. 
In contrast, in the distribution-free agnostic model~\cite{Haussler:92, KSS:94},
even {\em weak} learning is computationally hard~\cite{GR:06, FGK+:06short, Daniely16, Tiegel22}.
Due to this computational intractability, a significant branch of 
research has focused on agnostically learning halfspaces in the {\em distribution-specific} setting. 
Intuitively, the underlying structure of the data distribution can potentially be 
leveraged to obtain non-trivial efficient algorithms robust to adversarial label noise.

Here we focus on the well-studied task of agnostically learning halfspaces 
{\em when the underlying distribution on examples is assumed to be Gaussian}.
That is, we are given i.i.d.\ samples 
from a joint distribution $D$ on labeled examples $(\bx, y)$, 
where $\bx \in \R^n$ is the example and $y \in \R$ is the corresponding label, 
and the goal is to compute a hypothesis that is competitive with the best-fitting halfspace.
Moreover, we assume that the marginal $D_{\bx}$ on $\R^n$ 
is the standard Gaussian $\mathcal{N}(\mathbf{0}, \mathbf{I})$.
As we will explain subsequently, the distributional assumption
makes the learning problem computationally easier, 
as compared to the distribution-free setting. Interestingly, 
{\em even the Gaussian version of the problem
exhibits information-computation tradeoffs that we explore --- 
and essentially resolve --- in this paper.}

For concreteness, we introduce some notation followed by the
definition of the aforementioned problem.
For a boolean-valued hypothesis $h:\R^n\to \{\pm 1\}$ 
and a distribution $D$ supported on $\R^n \times \{\pm 1\}$,
we use $R_{0-1}(h; D)$ to denote the 0-1 error of $h$ with respect to $D$, i.e., 
$R_{0-1}(h; D) \eqdef \pr_{(\bx,y)\sim D}[h(\bx)\neq y]$.
For a class $\mathcal{C}$ of boolean-valued functions on $\R^n$, 
we use $R_{0-1}(\mathcal{C}; D)$ to denote the minimum 0-1 error 
of any $h \in \mathcal{C}$, i.e., 
$R_{0-1}(\mathcal{C}; D)  \eqdef \min_{h \in \mathcal{C}}R_{0-1}(h; D)$.

\begin{problem} [Agnostically Leaning Halfspaces under Gaussian Marginals]
\label{prob:LTF}
Let $\mathrm{LTF}$ be the class of halfspaces on $\R^n$.
Given an error parameter $0< \eps < 1$ and i.i.d.\ samples $(\bx, y)$ 
from a distribution $D$ on $\R^n \times \{\pm 1\}$, where
the marginal $D_{\bx}$ on $\R^n$ is the standard Gaussian $\mathcal{N}(\mathbf{0}, \mathbf{I})$
and no assumptions are made on the labels $y$, 
the goal of the learning algorithm $\mathcal{A}$ is to output a hypothesis $h: \R^n \to \{\pm 1\}$ 
such that $R_{0-1}(h; D) \leq R_{0-1}(\mathrm{LTF}; D) +\eps$ with high probability.
We will say that the algorithm $\mathcal{A}$ agnostically learns halfspaces (or LTFs) 
under Gaussian marginals to additive error $\eps$.
\end{problem}


\paragraph{Prior Work on Problem~\ref{prob:LTF}}
By standard results~\cite{Haussler:92, KSS:94}, it follows that the sample complexity 
of the agnostic learning problem for halfspaces is $O(n/\eps^2)$.
The $L_1$-regression algorithm of~\cite{KKMS:08} solves Problem~\ref{prob:LTF}
with sample complexity and running time $n^{O(1/\eps^2)}$~\cite{DGJ+10:bifh, DKN10}.
While the $L_1$-regression algorithm is not proper, recent work developed
a proper learner with qualitatively similar sample and time complexities 
(namely, $n^{\poly(1/\eps)}$)~\cite{DKKT21}.
Importantly, the $L_1$-regression algorithm remains
the most efficient known algorithm for the problem. 

Given the gap between the sample complexity of the problem 
and the complexity of known algorithms,
it is natural to ask whether the limitations of known efficient algorithms are inherent. 
There are two general approaches to establish {\em information-computation} tradeoffs 
for statistical problems. One approach focuses on restricted families 
of algorithms (e.g., Statistical Query algorithms or low-degree polynomial tests).
It should be noted that such results do not have any implications 
for the family of all polynomial-time algorithms. Another, arguably more convincing approach, 
is via efficient reductions from known (average-case) hard problems. 
This is the approach we adopt in this work.

Returning to Problem~\ref{prob:LTF}, 
a line of work~\cite{GGK20-agnostic-SQ, DKZ20, DKPZ21} 
has established tight hardness in the Statistical Query (SQ) model.
SQ algorithms~\cite{Kearns:98} 
are a class of algorithms that are only allowed 
to query expectations of bounded functions of the distribution 
rather than directly access samples.
\cite{DKPZ21} showed that any SQ algorithm for the problem
either requires $2^{n^{\Omega(1)}}$ queries
or at least one query of very high accuracy 
(suggesting a sample complexity lower bound of $n^{\Omega(1/\eps^2)}$).
Interestingly, it is known (see, e.g.,~\cite{DFTWW15}) that the $L_1$-regression algorithm can
be efficiently implemented in the SQ model. However,
since the SQ model is restricted, this SQ lower
bound has no implications for general efficient algorithms.

Prior to the our work, the only known {\em computational } hardness 
for Problem~\ref{prob:LTF} is due to Klivans and Kothari~\cite{KlivansK14}.
That work gave a reduction from the problem of 
learning sparse parities with noise to Problem~\ref{prob:LTF}.
Under the plausible assumption that learning $k$-sparse parities with noise over $\{0, 1\}^n$ 
requires time $n^{\Omega(k)}$, the reduction of \cite{KlivansK14} implies
a computational lower bound of $n^{\Omega(\log(1/\eps))}$ for Problem~\ref{prob:LTF}.
Interestingly, this lower bound cannot be improved in the sense that
the corresponding hard instances can be solved in time $n^{O(\log(1/\eps))}$.

Finally, we note that for the qualitatively weaker error guarantee
of $C \cdot \opt+\eps$, for a sufficiently large universal
constant $C>1$, $\poly(d/\eps)$ time algorithms 
are known~\cite{ABL17, Daniely15, DKS18a}.

In summary, the best known algorithm for Problem~\ref{prob:LTF}
has sample complexity and running time $n^{\poly(1/\eps)}$, while the best known
computational hardness result gives an $n^{\Omega(\log(1/\eps))}$ lower bound.
Moreover, a tight lower bound is known for the restricted class of SQ algorithms.
This raises the following natural question:
\begin{center}
{\em Can we establish a near-optimal {\em computational hardness} result for Problem~\ref{prob:LTF}?}
\end{center}
In this paper, we answer this question in the affirmative by exhibiting 
a computational hardness reduction from a classical cryptographic problem,
showing that current algorithms are essentially best possible.  
Specifically, we prove a complexity lower bound of $n^{\poly(1/\eps)}$ (Theorem~\ref{thm:inf-LTF}), 
assuming the widely believed sub-exponential hardness 
of the Learning with Errors (LWE) problem (Definition~\ref{def:general-lwe-one-dim}).

The task of learning halfspaces is as a special case
of the more general setting that the underlying function is of the form
$\sigma( \left(\langle \bw, \bx \rangle - t \right))$, where $\sigma: \R \to \R$ is
a univariate activation. If the activation is better behaved than 
the $\sign$ function, specifically if $\sigma$ is monotone and Lipschitz (aka the setting of 
Generalized Linear Models), then the learning problem can be easier computationally. 
Here we show that our techniques can be extended to prove near-optimal hardness for 
some of these cases as well. Specifically, we focus on the well-studied problem of ReLU regression.

A ReLU is any function 
$f: \R^n \to \R_+$ of the form
$f(\bx) = \relu \left(\langle \bw, \bx \rangle - t \right)$,
where $\bw \in \R^n$ is called the weight vector and $t \in \R$ is called the threshold.
The activation $\relu: \R \to \R_+$ is defined as $\relu(u)= \max \{ 0, u\}$.
ReLUs are the most commonly used activations in modern deep neural networks.
Moreover, finding the best-fitting ReLU with respect to square-loss
is a fundamental primitive in the theory of neural networks. 
A line of work studied this problem from the perspectives 
of both algorithms and lower bounds,
see, e.g.,~\cite{Mahdi17, GoelKKT17, MR18, GoelKK19, 
FCG20, DGKKS20, DKTZ22, ATV22}.
Similarly to the case of halfspaces, ReLU regression 
is efficiently solvable in the realizable
setting and computationally hard (even for weak error guarantees) 
in the distribution-independent
agnostic setting~\cite{MR18, DKMR22-agn}. 
Here we study the agnostic setting with Gaussian marginals.

Since ReLU regression is a real-valued task, 
we will require the analogous terminology.
For a real-valued hypothesis $h:\R^n\to \R$ 
and a distribution $D$ supported on $\R^n \times \{\pm 1\}$,
we use $R_{2}(h; D)$ to denote the $L_2^2$-error 
of $h$ with respect to $D$, i.e., 
$R_{2}(h; D) \eqdef \E_{(\bx,y) \sim D}[(h(\bx)-y)^2]$.
For a class $\mathcal{C}$ of real-valued functions on $\R^n$, 
we use $R_{2}(\mathcal{C}; D)$ to denote the minimum 
$L_2^2$-error of any $h \in \mathcal{C}$, i.e., 
$R_{2}(\mathcal{C}; D)  \eqdef \min_{h \in \mathcal{C}}R_{2}(h; D)$.

\begin{problem} [ReLU Regression under Gaussian Marginals]
\label{prob:ReLU}
Let $\mathrm{ReLU}$ be the class of ReLUs on $\R^n$ with weight vectors
in the set $\{\bw \in \R^n: \|\bw\|_2 \leq 1\}$.
Given an additive error parameter $0< \eps < 1$ and i.i.d.\ samples $(\bx, y)$ 
from a distribution $D$ on $\R^n \times \R$, where
the marginal $D_{\bx}$ on $\R^n$ 
is the standard Gaussian $\mathcal{N}(\mathbf{0}, \mathbf{I})$
and the labels $y$ are bounded, 
the goal of the learning algorithm $\mathcal{A}$ is to output a hypothesis $h: \R^n \to \R$ 
such that $R_{2}(h; D) \leq R_{2}(\mathrm{ReLU}; D) +\eps$ with high probability.
We will say that the algorithm $\mathcal{A}$ agnostically learns ReLUs
under Gaussian marginals to additive error $\eps$.
\end{problem}


\paragraph{Prior Work on Problem~\ref{prob:ReLU}}
While there is no black-box relation with Problem~\ref{prob:LTF}, the situation
for both problems is analogous.
\cite{DGKKS20} gave an algorithm for Problem~\ref{prob:ReLU} with
sample complexity and runtime $n^{\poly(1/\eps)}$. 
While $\poly(n/\eps)$ time algorithms are known 
with weaker guarantees~\cite{GoelKK19, DGKKS20,  DKTZ22},
the fastest known algorithm with $\opt+\eps$ error 
is the one of~\cite{DGKKS20}.
In terms of computational hardness,~\cite{GoelKK19} gave 
a reduction from sparse noisy parity implying a computational 
lower bound of $n^{\Omega(\log(1/\eps))}$ for Problem~\ref{prob:ReLU}.
In the restricted SQ model, (near-optimal) SQ lower bounds of $n^{\poly(1/\eps)}$
have been shown~\cite{GGK20-agnostic-SQ, DKZ20, DKPZ21}.

In summary, the best known algorithm for Problem~\ref{prob:ReLU}
has sample complexity and running time $n^{\poly(1/\eps)}$, while the best known
computational hardness result gives an $n^{\Omega(\log(1/\eps))}$ lower bound.
It is thus natural to ask whether {\em computational hardness} of 
$n^{\poly(1/\eps)}$ can be established. Similarly to the case of LTFs,
we prove such a statement (Theorem~\ref{thm:inf-ReLU}) 
under the sub-exponential hardness of LWE.

\subsection{Our Results and Techniques} \label{ssec:results}
We start with an informal definition of the LWE problem.
In the LWE problem,
we are given samples $(\bx_1, y_1), \dots, (\bx_m, y_m)$ 
and the goal is to distinguish between the following two cases:
\begin{itemize}[leftmargin=*]
\item Each $\bx_i$ is drawn uniformly at random (u.a.r.)\ from $\Z^n_q$, 
and there is a hidden secret vector $\bs \in \Z_q^n$ such that 
$y_i = \left<\bx_i, \bs\right> + \nsv_i$, 
where $\nsv_i \in \Z_q$ is discrete Gaussian noise (independent of $\bx_i$).
\item Each $\bx_i$ and each $y_i$ are independent 
and are sampled u.a.r. from $\Z_q^n$ and $\Z_q$ respectively.
\end{itemize}


Formal definitions of LWE (Definition~\ref{def:general-lwe-one-dim}) 
together with the precise computational hardness assumption (Assumption~\ref{asm:LWE-hardness}) we rely on 
are given in Section~\ref{sec:prelims}. 

For Problem~\ref{prob:LTF} we prove:

\begin{theorem}[Hardness of Agnostically Learning Gaussian Halfspaces]\label{thm:inf-LTF}
Assume that LWE cannot be solved in $2^{n^{1-\Omega(1)}}$ time. 
Then for any constants $c>0$ and $\alpha < 2$ the following holds: 
If $\eps \leq 1/\log^{1/2+c}(n)$,
any algorithm that agnostically learns LTFs on $\R^n$ 
with Gaussian marginals to additive error $\eps$ 
requires running time at least $\min \{ n^{\Omega(1/(\eps\sqrt{\log n})^\alpha)}, 2^{n^{0.99}} \}$.
\end{theorem}

Some comments are in order to interpret this statement.
The minimum of the two terms is necessary to handle the case
where $\eps$ is very small, specifically $\eps = \tilde{O}(1/\sqrt{n})$. 
(Since the problem can always be solved in time $2^{\tilde{O}(n)}$ 
via brute-force, the first term cannot be a time lower bound for this range of $\eps$.)
On the other hand, for $ \tilde{\Omega}(1/\sqrt{n})  = \eps  \leq 1/\log^{1/2+c}(n)$, 
Theorem~\ref{thm:inf-LTF} gives a time lower bound of 
$n^{\Omega(1/(\eps\sqrt{\log n})^\alpha)}$, for any constant $\alpha <2$.
This bound nearly matches the upper bound of $n^{O(1/\eps^2)}$~\cite{KKMS:08}, 
up to the $\sqrt{\log n}$ factor
in the exponent. Note that the extraneous factor of $\sqrt{\log n}$ 
is negligible if $\eps$ is sufficiently small.
For example, if $\eps \leq 1/\log n$, the implied lower bound is $n^{\Omega(1/\eps^\alpha)}$
for any constant $\alpha<1$. For $\eps  = O(n^{-c})$, for a small constant $c>0$, 
we get a lower bound of $n^{\tilde{\Omega}(1/\eps^\alpha)}$, for any constant $\alpha <2$.

For Problem~\ref{prob:ReLU} we prove:

\begin{theorem}[Hardness of Gaussian ReLU Regression]\label{thm:inf-ReLU}
Assume that LWE cannot be solved in $2^{n^{1-\Omega(1)}}$ time. 
Then for any constants $c>0$ and $\alpha < 1/2$ the following holds: 
If $\eps \leq 1/\log^{2+c}(n)$,
any algorithm for ReLU regression on $\R^n$ under Gaussian marginals
with additive error $\eps$ requires running time at least 
$\min \{ n^{\Omega(1/(\eps \log^2 n)^\alpha)}, 2^{n^{0.99}} \}$.
\end{theorem}

Intuitively, the above statement says that any algorithm for Problem~\ref{prob:ReLU}
requires time at least $n^{(1/\eps)^{\Omega(1)}}$, if $\eps$ is sufficiently small 
(e.g., $\eps = O(1/\log^3n)$) and not too small (in which case the latter term dominates
the obvious brute-force algorithm). This runtime lower bound qualitatively matches the
upper bound of $n^{\poly(1/\eps)}$~\cite{DGKKS20} 
and exponentially improves on the best known computational 
lower bound of $n^{\Omega(\log(1/\eps))}$~\cite{GoelKK19}.


\subsection{Techniques} \label{ssec:tec}
Our computational hardness reductions build on two main ideas. 
The first idea is inspired by the approach of \cite{DKMR22}.
We note that \cite{DKMR22} established a hardness reduction
from LWE to {\em distribution-free} PAC learning halfspaces with Massart noise.
While the Massart noise model is technically easier than the adversarial label noise model,
here we are interested in the (much simpler)
regime where the marginal distribution is Gaussian.
Indeed, the results of  \cite{DKMR22} have no implications for the Gaussian setting.
Yet one of their ideas is useful in our context.

The key idea of \cite{DKMR22} is that by applying rejection sampling 
to a continuous variant of LWE supported on $\R^n$
(this variant was shown to be as hard as the 
standard LWE problem supported on $\Z_q^n$ in~\cite{vinod2022})
one obtains either 
(i) a standard Gaussian in the null hypothesis case or 
(ii) a distribution that is approximately a discrete Gaussian plus a little noise 
in a hidden direction and a standard Gaussian in the orthogonal directions in the alternative hypothesis case. 
By taking a mixture of such rejection sampling distributions, 
\cite{DKMR22} manage to produce a joint distribution on $(\bx,y)$ over $\R^n \times \{\pm 1\}$ such that: 
\begin{enumerate}[leftmargin=*]
\item[(i)] in the null hypothesis case, 
$y$ is independent of $\bx$, and
\item[(ii)] in the alternative hypothesis case\footnote{This leverages a 
construction of such a distribution from~\cite{DK20-hardness}.}, 
$y$ is given by a Polynomial Threshold Function (PTF) applied to $\bx$ with Massart noise.
\end{enumerate}
Given the above, \cite{DKMR22} conclude that
any learner for Massart halfspaces LTFs can be used
to distinguish between the alternative and null hypothesis cases, 
and thus solves the LWE problem.

In this paper, we apply a similar technique to the tasks of agnostically leaning halfspaces
and ReLUs under Gaussian marginals. A key difference in our setting is that we require the distribution 
of $\bx$ be the standard Gaussian --- a property inherently not satisfied by the aforementioned construction.
Roughly speaking, \cite{DKMR22} showed that it is LWE-hard to distinguish 
between a standard Gaussian and a distribution that is standard Gaussian 
in all directions except for a hidden direction in which 
it is approximately a specified mixture of discrete Gaussians plus a little noise. 
The learning application in \cite{DKMR22} was obtained via the construction 
of a PTF with Massart noise such that both the conditional distributions on $y=1$ 
and on $y=-1$ were such (noisy) mixtures of discrete Gaussians. 
In our context, we need to construct different pairs of such conditional distributions.

We do this as follows.
Let $\bx$ be sampled from a standard Gaussian and 
consider the function $f_{\bs}(\bx) = (-1)^{\lfloor\langle  \bx,\bs\rangle  \rfloor}$ 
for some unknown vector $\bs$ with relatively large norm. 
If we consider the distribution of $\bx$
conditioned on $f_{\bs}(\bx) = 1$, we obtain 
a distribution that is (i) Gaussian in the directions orthogonal to $\bs$, 
and (ii) a Gaussian conditioned on $\lfloor\langle  \bx,\bs\rangle  \rfloor$ being even 
in the $\bs$-direction. 
One can see that this is a mixture of discrete Gaussians.
The same can be argued for the distribution of $\bs$ conditioned on $f_{\bs}(\bx) = -1$. 
Thus, using the techniques described above, 
we can show that given labeled samples $(\bx,y)$ with $\bx$ a standard Gaussian, 
it is LWE-hard to distinguish between the cases that 
(i) $y$ is independent of $\bx$, and  
(ii) $y = f_{\bs}(\bx)$ for some unknown vector $\bs$.

This result forms the basis for our two learning applications. 
Specifically, for the problem of agnostically learning Gaussian LTFs, 
it is not hard to show that there exists an LTF $g$ 
such that $\E_{\bx \sim \mathcal{N}(\mathbf{0}, \mathbf{I})}[f_{\bs}(\bx)g(\bx)] =\eps =\Omega(1/\|\bs\|_2)$. 
This implies that any algorithm that agnostically learns LTFs 
to error $\opt+ \eps/3$, where $\opt = R_{0-1}(\mathrm{LTF}; D)$, 
can be used to distinguish between 
the case that $y$ is independent of $\bx$ (in which case $\opt = 1/2$) 
and the case described above (i.e., $y=f_{\bs}(\bx) = (-1)^{\lfloor\langle  \bx,\bs\rangle  \rfloor}$), 
where $\opt = 1/2-\eps$. This implies that the agnostic learning of Gaussian LTFs is LWE-hard.

For ReLU regression, we show
that there exists a ReLU $g$ such that 
$\E_{\bx \sim \mathcal{N}(\mathbf{0}, \mathbf{I})}[f_{\bs}(\bx)g(\bx)] =\eps=\Omega(1/\|\bs\|_2^2)$. 
In particular, this correlation means that the $L^2_2$-distance 
between $f$ and an appropriately scaled version of $g$ 
is bounded away from $1$ in the negative direction. 
Thus, it is LWE-hard to distinguish between the case where $y=f_{\bs}(\bx)$ 
(and thus the minimal $L^2_2$-error for ReLUs is at most $1-\eps^2$) 
and the case where $y$ is independent of $\bx$ 
(in which case the minimum $L^2_2$-error of any ReLU is at least $1$).

The above sketch glossed over the following important technical point.
By applying the aforementioned reduction directly to the standard version 
of the (continuous) LWE problem~\cite{BRST21}
which has secret vector $\bs$ with $\|\bs\|_2=\sqrt{n}$,
we can obtain a time lower bound for our agnostic learning problems 
{\em only if the additive error $\eps$ is tiny}, namely $\eps= \tilde{O}(1/\sqrt{n})$.
In order to prove lower bounds for a wider
range of $\eps$, 
we will need to instead start from a {\em small norm version} 
of the continuous LWE problem, where the secret vector $\bs$ 
roughly satisfies $\|\bs\|_2\approx 1/\eps$.
We accomplish this via a non-trivial modification 
of a reduction in \cite{vinod2022}, which we view
as an additional technical contribution of this work. Specifically, 
\cite{vinod2022} gave a reduction of the standard discrete LWE problem 
to a discrete LWE problem with a sparse secret 
(namely, secret vector $\bs\in  \{0, \pm 1\}^n$ with $\|\bs\|_1=k$). 
(This itself leverages an idea in \cite{mic18}.)
After that, \cite{vinod2022} further reduces the sparse secret discrete LWE problem 
to a continuous LWE problem whose secret vector has small $\ell_2$-norm.
The limitation here is that their $\ell_2$-norm bound has a factor of $\sqrt{\log m}$, 
where $m$ is the number of samples. 
Unfortunately, this quantitative dependence prevents 
us from obtaining the near optimal lower bound for our learning LTFs tasks. 
To address this issue, we present a (slightly) improved reduction 
(see Lemma \ref{lem:lwe-gaussian-sample}), removing the $\sqrt{\log m}$ factor on the secret vector norm. 
This allows us to apply our reduction technique to the small norm continuous LWE problem,
giving nearly tight lower bounds for our learning problems.

\section{Preliminaries} \label{sec:prelims}

\paragraph{Notation}
We use $\langle\bx,\by \rangle$ for the inner product 
between vectors $\bx, \by \in \R^n$.
For $p \geq 1$ and $\bx\in \R^n$, 
we use $\|\bx\|_p\eqdef\left (\sum_{i=1}^n |\bx_i|^p\right )^{1/p}$ 
to denote the $\ell_p$-norm of $\bx$.
We use $\mathbb{S}^{n-1}$ to denote the unit sphere in $\R^n$, i.e., the 
set $\mathbb{S}^{n-1} \eqdef \{\bx \in\R^n : \|\bx\|_2=1 \}$.
For $q \in \N$, we denote $\Z_q \eqdef \{0,1,\cdots,q-1\}$ 
and $\R_q \eqdef [0,q)$. 
We use $\mod_{q}: \R^n \to \R_q^n$ 
to denote the function that applies the $\mod_q$ operation 
on each coordinate of the vector $\bx$. 
For a set $S\subset \R^n$, we use $U(S)$ to denote the uniform distribution over $S$.
We use $\bX\sim D$ to denote a random variable $\bX$ with distribution $D$. 
For a random variable $\bx$ (resp. a distribution $D$), we use $P_\bx$
(resp. $P_D$) to denote the probability density function or probability 
mass function of the random variable $\bx$ 
(resp. distribution $D$).
We will require the following notion of partially supported Gaussians.

\begin{definition} [Partially Supported Gaussian Distribution] \label{def:dpart}
For $\sigma\in \R_{+}$ and $\bx\in \R^n$, let
$\rho_{\sigma}(\bx)\eqdef \sigma^{-n}\exp\left (-\pi(\|\bx\|_2/\sigma)^2\right )$.
For any countable set\footnote{We will take the sets $S$ to be shifts of lattices, guaranteeing  
that $\rho_{\sigma}(S)$ is finite and the distribution is well-defined.} $S \subseteq \R^n$, 
we let $\rho_{\sigma}(S)\eqdef \sum_{\bx\in S} \rho_{\sigma}(\bx)$, 
and let $\Dgaus_{S,{\sigma}}$ be the distribution supported on $S$ with pmf
$P_{\Dgaus_{S,\sigma}}(\bx) = \rho_{\sigma}(\bx)/\rho_{\sigma}(S)$.
\end{definition}
For consistency, 
we will use $D^\gaus_{\R^n,\sqrt{2\pi}\sigma}$ to denote the $n$-dimensional
Gaussian distribution $\gaus(\mathbf{0}, \sigma^2\mathbf{I})$.

\paragraph{Learning with Errors}
The Learning with Errors (LWE) problem was introduced in \cite{reg05}. 
Here we use a slightly more generic definition
for the convenience of later reductions between different variants of LWE problems.

\begin{definition} [Generic LWE]\label{def:general-lwe-one-dim}
Let $m,n\in \N$, $q\in \R_+$, and let 
$D_\smp, D_\scr, D_\ns$ 
be distributions on $\R^n, \R^n, \R$ respectively. 
In the $\lwe(m,D_\smp,D_\scr,D_\ns,\mod_q)$ problem, 
we are given $m$ independent samples $(\bx,y)$ 
and want to distinguish between the following two cases:
\begin{enumerate} [leftmargin=*]
	\item [(i)] {\bf Alternative hypothesis}: A vector $\bs$ is drawn from $D_\scr$ ($\bs$ is called ``the secret vector''). 
	Then each sample $(\bx,y)$ is generated by taking $\bx\sim D_\smp, \nsv\sim D_\ns$, 
	and letting $y= \mod_q(\left<\bx, \bs\right>+\nsv)$.  	
	\item [(ii)] {\bf Null hypothesis}: The random variables $\bx$ and $y$ 
	are independent.
	Moreover, $\bx$ has the same marginal distribution as in the alternative hypothesis,
	and $y$ has the marginal distribution as $U(S)$ 
	where $S$ is	the support of the marginal distribution of $y$ in the alternative hypothesis.
\end{enumerate}
An algorithm $A$ solves the LWE problem with advantage $\alpha>0$, if 
$p_{\rm alternative}-p_{\rm null}\geq \alpha$
where $p_{\rm alternative}$ (resp. $p_{\rm null}$) is the probability 
that $A$ outputs ``alternative hypothesis'' 
if the input distribution is from the alternative hypothesis (resp. null hypothesis).
When a distribution in LWE 
is uniform over some set $S$, 
we may abbreviate $U(S)$ as $S$.
\end{definition}

Our hardness assumption is the following: 

\begin{assumption} [Sub-exponential LWE Assumption] \label{asm:LWE-hardness}
Let $c > 0$ be a sufficiently large constant and $q\in \N$. 
For any constants $\beta\in (0,1)$, $\kappa \in \N$, the problem
$\lwe(2^{O(n^\beta)},\Z_q^n,\Z_q^n ,\Dgaus_{\Z, \sigma},\mod_q)$ 
with $q \leq n^\kappa$ and $\sigma = c\sqrt{n}$ 
cannot be solved in $2^{O(n^\beta)}$ time with $2^{-O(n^\beta)}$ advantage.
\end{assumption}

This is a widely-believed conjecture, supported by our current understanding of the field.
\cite{reg05,Peikert09} 
gave a polynomial-time \emph{quantum} reduction
from approximating (the decision version of) 
the Shortest Vector Problem (GapSVP) to LWE (with similar $n, q, \sigma$ parameters). 
We note that
the fastest known algorithm for GapSVP takes $2^{O(n)}$ time~\cite{AggarwalLNS20}.
Thus, refuting the conjecture would be a major breakthrough. 
A similar assumption was also used in~\cite{vinod2022}  and~\cite{DKMR22}
to establish computational hardness of learning Gaussian mixtures
and distribution-independent learning of Massart halfspaces. 

In addition to the standard LWE problem above, we will also consider 
a continuous variant of the LWE problem (introduced in \cite{BRST21}) 
where supports of the distributions are continuous.
In particular, the first part of our proof is the following proposition which 
slightly modifies the proof in \cite{vinod2022} and gives
the reduction from the standard LWE to the continuous LWE. 
The proof is deferred to Appendix \ref{app:clwe}. 

\begin{proposition} [Hardness of continuous LWE (cLWE) with Small-Norm Secret] \label{pro:lwe-short-norm}
Under Assumption \ref{asm:LWE-hardness}, 
for any $n\in \N$, any constants $\beta\in (0,1)$, $\kappa\in \N$, $\gamma\in \R_+$ and
any $\log^\gamma n\leq k\leq cn$ where $c>0$ is a sufficiently small universal constant, 
the problem $\lwe(n^{O(k^\beta)},D_{\R^n,1}^\gaus,\mathbb{S}^{n-1},D_{\R,\sigma}^\gaus,\mod_T)$ with 
$\sigma\geq k^{-\kappa}$ and $T=1/(c'\sqrt{k\log n})$ where $c'>0$ is a sufficiently large universal constant 
cannot be solved in $n^{O(k^\beta)}$ time with $n^{-O(k^\beta)}$ advantage.
\end{proposition}

\section{Hardness of Agnostically Learning Gaussian LTFs} \label{sec:ltfs}

In this section, we continue from Proposition \ref{pro:lwe-short-norm} 
(the proof of which is deferred to Appendix \ref{app:clwe}) 
which is the first step of our reduction,
and give the second and main part of the reduction.
We thereby establishing 
the desired cryptographic hardness of agnostically learning LTFs under the Gaussian distribution. 

The high-level idea is the following.
Given samples $(\bx,y)$ from a distribution $D$ on $\R^n\times \R_T$, 
which is an instance of 
the cLWE problem $\lwe(n^{O(k^\beta)},D_{\R^n,1}^\gaus,\mathbb{S}^{n-1},D_{\R,\sigma}^\gaus,\mod_T)$
(note that $T$ is the ``period'' of the periodic signal on the hidden direction) from
Proposition \ref{pro:lwe-short-norm},
we efficiently generate samples $(\bx,y')$ (we leave $\bx$ unchanged) 
from a distribution $D'$ on $\R^n\times \{\pm 1\}$ such that:
\begin{enumerate}  
	\item [(i)] If $D$ is from the alternative hypothesis case, then there exists an LTF 
	$h:\R^n\to \{\pm 1\}$ such that $R_{0-1}(h; D')\leq 1/2-\Omega(T)$.
	
	\item [(ii)] If $D$ is from the null hypothesis case, then for $(\bx,y')\sim D'$, we have that 
	$y'=+1$ with probability $1/2$ and $y'=-1$ with probability $1/2$ independent of $\bx$; 
	thus, no hypothesis can achieve error non-trivially better than $1/2$.
\end{enumerate}  
Given the above properties, if an algorithm can agnostically learn 
LTFs with Gaussian marginals 
to error $R_{0-1}(\LTF; D')+o(T)$, then it can distinguish 
the two cases above and solve the LWE problem.

In the body of this section, we describe our reduction and formalize the above.
The main theorem of this section, stated and proved below, 
establishes hardness for a natural decision version 
of agnostically learning LTFs.

\begin{theorem}[Cryptographic Hardness of Agnostically Learning Gaussian LTFs] \label{thm:main-thm-ltf}
Under Assumption \ref{asm:LWE-hardness}, 
for any $n\in \N$,
for any constants $\beta\in (0,1)$, $\gamma\in \R_+$ and any $\log^{\gamma} n\leq k\leq cn$
where $c$ is a sufficiently small constant,
there is no algorithm that runs in time $n^{O(k^\beta)}$ and distinguishes between 
the following two cases of a joint distribution $D$ of $(\bx,y)$ supported on $\R^n\times \{\pm 1\}$ 
with marginal $D_\bx=D^\gaus_{\R^n,1}$, with $n^{-O(k^\beta)}$ advantage:
\begin{enumerate}[leftmargin=*]
\item [(i)] {\bf Alternative Hypothesis:} There exists an LTF with 0-1 error non-trivially smaller than $1/2$, namely
$R_{0-1}(\LTF; D)\leq 1/2-\Omega\left (1/\sqrt{k\log n}\right )$.
\item [(ii)] {\bf Null Hypothesis:} A sample $(\bx, y) \sim D$ satisfies the following:
$y=+1$ with probability $1/2$ and $y=-1$ with probability $1/2$ independent of $\bx$. 
\end{enumerate}
\end{theorem}
\begin{proof}
We give an efficient method taking as input 
samples from a distribution $D'$ --- that is either from the alternative hypothesis
or the null hypothesis of 
$\lwe(n^{O(k^\beta)},D_{\R^n,1}^\gaus,\mathbb{S}^{n-1},D_{\R,\sigma}^\gaus,\mod_T)$ 
from Proposition \ref{pro:lwe-short-norm} --- 
and generates samples from another distribution $D$ with the following properties:
If $D'$ is from the alternative (resp.~null) hypothesis of the LWE problem, 
then the resulting distribution $D$ will satisfy the alternative (resp.~null) 
hypothesis requirement of the theorem for the agnostic LTF learning decision problem.

The reduction process is the following:
For a sample $(\bx,y')$ from a distribution $D'$, 
which is an instance of the problem
$\lwe(n^{O(k^\beta)},D_{\R^n,1}^\gaus,\mathbb{S}^{n-1},D_{\R,\sigma}^\gaus,\mod_T)$ 
from Proposition \ref{pro:lwe-short-norm},
we simply output $(\bx,y) \sim D$, where $y=+1$ if $y'\leq T/2$ and $y=-1$ otherwise.
We argue that $D$ satisfies the desired requirements stated above.
We first note that the marginal $D_{\bx}$ of $D$ satisfies $D_{\bx}=D^\gaus_{\R^n,1}$, 
therefore it suffices to verify that
$R_{0-1}(\LTF; D)=1/2-\Omega\left (1/\sqrt{k\log n}\right )$ and 
$y=+1$ with probability $1/2$
independent of $\bx$
for each case respectively.

For the alternative hypothesis case,
let $D'$ be from the alternative hypothesis case of the LWE.
Let $\bs$ be the secret vector in the LWE problem.
We consider the following two LTFs: 
$h_1(\bx)=\sign(\langle \bs,\bx\rangle -T/6)$
and $h_2(\bx)=\sign(-\langle \bs,\bx\rangle +T/3)$.
If we can show that $R_{0-1}(h_1; D)+R_{0-1}(h_2; D)\leq 1-\Omega(T)$,
then either $h=h_1$ or $h=h_2$ satisfies $R_{0-1}(h; D)\leq 1/2-\Omega(T)$,
which implies that $R_{0-1}(\LTF; D)\leq R_{0-1}(h; D)\leq 1/2-\Omega(1/\sqrt{k\log n})$
by the definition of $T$.  

To show that $R_{0-1}(h_1; D)+R_{0-1}(h_2; D)\leq 1-\Omega(T)$, 
we examine the subset of the domain where $h_1$ and $h_2$ agree,
namely the region 
\begin{align*}
B\eqdef &\{\bt\in \R^n\mid h_1(\bt)=h_2(\bt)\}\\
=&\{\bt\in \R^n\mid \langle \bs,\bt \rangle \in [T/6,T/3]\} \;.
\end{align*}
Since for any $\bt\in B$, it is always the case that $h_1(\bt)=h_2(\bt)=+1$,
we can write
\begin{align*}
	&R_{0-1}(h_1; D)+R_{0-1}(h_2; D)  \\
	=&\pr_{(\bx,y)\sim D}[y\neq h_1(\bx)]+\pr_{(\bx,y)\sim D}[y\neq h_2(\bx)]\\
	=&\pr_{(\bx,y)\sim D}[\bx\not\in B\land y\neq h_1(\bx)]\\
	&+\pr_{(\bx,y)\sim D}[\bx\not\in B\land y\neq h_2(\bx)]\\
	&+2\pr_{(\bx,y)\sim D}[\bx\in B\land y=-1] \;.
\end{align*}
Since for any $\bx\not \in B$ we have that $h_1(\bx)\neq h_2(\bx)$, 
the first two terms sum to $\pr_{(\bx,y)\sim D}[\bx\not\in B]$.
Therefore, we have that
\begin{align*}
	&R_{0-1}(h_1; D))+R_{0-1}(h_2; D) \\
	=& \pr_{(\bx,y)\sim D}[\bx\not\in B]+2\pr_{(\bx,y)\sim D}[\bx\in B\land y=-1]\\
	=& 1+\pr_{(\bx,y)\sim D}[\bx\in B\land y=-1]\\
	&-\pr_{(\bx,y)\sim D}[\bx\in B\land y=+1]\\
	=& 1-\pr[\bx\in B](1-2\pr_{(\bx,y)\sim D}[y=-1 \mid \bx\in B]) \;.
\end{align*}
From the definition of $B$ and $\bx\sim D^\gaus_{\R^n,1}$, 
we have $\pr[\bx\in B]=\Omega(T)$. Thus, we obtain
\begin{equation} \label{eqn:sum-err}
R_{0-1}(h_1; D)+R_{0-1}(h_2; D)
= 1-\Omega(T) \left( 1-2\pr_{(\bx,y)\sim D}[y=-1 \mid \bx\in B] \right) \;.
\end{equation}
If we can show that $\pr[y=-1 \mid \bx\in B]\leq 1/3$, then we are done
since this implies that $R_{0-1}(h_1; D)+R_{0-1}(h_2; D)\leq 1-\Omega(T)$.

We note that from the definition of the Alternative case distribution 
of the LWE problem, we have
$$y'=\mod_{T}(\langle \bs,\bx\rangle +z) \;,$$
and that $y=-1$ only if $y'>T/2$,
which in turn happens only if 
$$\langle \bs,\bx\rangle +z> T/2 \textrm{ or } \langle \bs,\bx\rangle +z< 0 \;.$$
For $\bx\in B$,  we have that $\langle \bs,\bx\rangle\in [T/6,T/3]$,
therefore $y=-1$ only if $|z|\geq T/6$.
Notice that $z\sim D_{\R,\sigma}^\gaus$
and Proposition \ref{pro:lwe-short-norm} states that the LWE problem 
is hard for any fixed constant $\kappa\in \N$ and $\sigma\geq k^{-\kappa}$.
Given the constant $\gamma\in \R^+$ in this theorem, 
we will take $\kappa=\lceil 1/(2\gamma)+1/2+1\rceil$ which is a fixed constant.
Then, by Proposition \ref{pro:lwe-short-norm},
the LWE problem is hard for $\sigma=k^{-\kappa}=1/(k^{3/2}\sqrt{\log n})=o(T)$.
Therefore, we have that
\begin{align*}
\pr_{(\bx,y)\sim D}[y=-1 \mid \bx\in B]
\leq \pr_{z\sim D^\gaus_{\R,\sigma}}[|z|\geq T/6]
=o(1)\; .
\end{align*}
Thus, plugging the above back to \eqref{eqn:sum-err}, we can conclude that
\begin{align*}
R_{0-1}(h_1; D)+R_{0-1}(h_2; D)
=1-\Omega(T) \left(1-2\pr_{(\bx,y)\sim D}[y=-1 \mid \bx\in B] \right)
\leq 1-\Omega(T) \;.
\end{align*}
Then, as argued above, 
if both 
$h=h_1$ and $h=h_2$ do not satisfy $R_{0-1}(h; D)\leq 1/2-\Omega\left (T\right )$,
then $R_{0-1}(h_1; D)+R_{0-1}(h_2; D)>1-\Omega\left (T\right )$, a contradiction.
Thus, either 
$h=h_1$ or $h=h_2$ satisfies 
$R_{0-1}(h; D)\leq 1/2-\Omega(T)\leq 1/2-\Omega\left (1/\sqrt{k\log n}\right )$.
This completes the proof for the alternative hypothesis case.

For the null hypothesis case, 
it is immediate that $y=+1$ with probability $1/2$ independent of $\bx$,
since $y'\sim U([0,T))$ independent of $\bx$ 
in the null hypothesis case of the LWE problem.
This completes the proof of correctness.

It remains to verify the time lower bound and the distinguishing 
advantage for agnostically learning LTFs.
From Proposition \ref{pro:lwe-short-norm}, 
we know that under Assumption \ref{asm:LWE-hardness},
for the problem
$\lwe(n^{O(k^\beta)},D_{\R^n,1}^\gaus,\mathbb{S}^{n-1},D_{\R,\sigma}^\gaus,\mod_T)$
with 
any $\sigma\geq k^{-\kappa}$ (where $\kappa\in \N$ is a constant) 
and $T=1/(c'\sqrt{k\log n})$, where $c'>0$ is a sufficiently large universal constant, 
the problem cannot be solved in $n^{O(k^\beta)}$ time with $n^{-O(k^\beta)}$ advantage.
Therefore, under the same assumption, there is no algorithm that solves the decision version of
the agnostic learning LTFs problem
(defined in the theorem statement) 
in time $n^{O(k^\beta)}$ with $n^{-O(k^\beta)}$ advantage.
\end{proof}

The following corollary immediately 
follows from Theorem~\ref{thm:main-thm-ltf}.

\begin{corollary}\label{cor:ltf}
Under Assumption \ref{asm:LWE-hardness}, 
for any constants $\alpha\in (0,2)$, $\gamma>1/2$
and any $c/(\sqrt{n\log n})\leq \eps\leq 1/\log^{\gamma} n$
where $c$ is a sufficiently large constant, 
there is no algorithm that agnostically learns LTFs on $\R^n$ 
with Gaussian marginals to additive error $\eps$ 
and runs in time $n^{O(1/(\eps\sqrt{\log n})^\alpha)}$.
\end{corollary}

\begin{proof}
We chose the parameter $k$ in Theorem \ref{thm:main-thm-ltf} to be the value that 
$\eps=c/\sqrt{k\log n}$, where $c$ is a sufficiently small constant.
Then any algorithm that agnostically learns LTFs to additive error $\eps$ 
can solve the testing problem of Theorem \ref{thm:main-thm-ltf} with probability $2/3$.
Therefore, no such algorithm should run in time $n^{O(k^\beta)}$ for any $\beta\in (0,1)$.
Since $\eps=c/\sqrt{k\log n}$, and if we chose $\beta=\alpha/2$, then
the time lower bound can be rewritten as 
$n^{O(k^\beta)}=n^{O(1/(\eps\sqrt{\log n})^{2\beta})}=n^{O(1/(\eps\sqrt{\log n})^\alpha)}$.
This completes the proof.
\end{proof}

\section{Hardness of ReLU Regression with Gaussian Marginals}
\label{sec:relus}

In this section, we establish near-optimal computational 
hardness for ReLU regression under Gaussian marginals.
It is worth pointing out that this hardness result would also apply
to any $L$-Lipschitz activation function $f:\R\to \R$, for $L = O(1)$, 
such that there exists $t\in \R$ so that $f(x)$ is a constant for any $x\leq t$.
Roughly, our result says that any algorithm that solves this problem
to error $\opt+\eps$ with Gaussian marginals 
requires $n^{\poly(1/(\eps\log^2 n))}$ time. 

The idea is to show that the same hard instance as in Section~\ref{sec:ltfs} 
can be distinguished by a ReLU regression algorithm.
The main theorem of this section, stated and proved below, 
establishes hardness for a natural decision version of agnostically learning ReLU.

\begin{theorem} \label{thm:main-thm-relu}
Under Assumption \ref{asm:LWE-hardness}, 
for any constants $\beta\in (0,1)$, $\gamma\in \R_+$ and
any $\log^\gamma n\leq k\leq cn$,
where $c$ is a sufficiently small constant,
there is no algorithm that runs in time $n^{O(k^\beta)}$ 
and distinguishes between the following two cases of joint distribution
$D$ on $(\bx,y)$ supported on $\R^n\times \{\pm 1\}$ 
with marginal $D_\bx=D^\gaus_{\R^n,1}$, with $n^{-O(k^\beta)}$ advantage:
\begin{enumerate}
\item [(i)] {\bf Alternative Hypothesis:} 
There exists a ReLU with $L_2^2$-error non-trivially smaller than $1$, namely
$R_{2}(\mathrm{ReLU}; D)\leq 1-\Omega\left (1/(k\log n)^2\right )$.
\item [(ii)] {\bf Null Hypothesis:} 
A sample $(\bx, y) \sim D$ satisfies the following:
$y=+1$ with probability $1/2$ and $y=-1$ with probability $1/2$ independent of $\bx$. 
\end{enumerate}
\end{theorem}

\begin{proof}
We start with the following intermediate lemma.
The lemma roughly says that if there exists a ReLU 
nontrivially correlated with a distribution, 
then there must be another ReLU with nontrivial $L_2^2$-error.  

\begin{lemma} \label{lem:correlation-implies-error}
Let $\eps\in (0,1)$ and $D$ be a joint distribution of 
$(u,y)$ supported on $\R\times \{\pm 1\}$ such that
the marginal $D_{u}=\D_{\R,1}^\gaus$ and $\E_{(u,y)\sim D}[y]=0$.
Suppose there is a ReLU of the form $f(u)=\relu(u-t)$ such that
$t\geq 0$ and $\left |\E_{(u,y)\sim D}[yf(u)]\right |\geq\eps$.
Then there exists $k\in (-1,1)$ such that
the ReLU $g(u)= \relu(ku-kt)$
satisfies
$\E_{(u,y)\sim D}[(y-g(u))^2]\leq1-\eps^2$.
\end{lemma}

\begin{proof}
We first note that $g(u)=kf(u)$, thus
\begin{align*}
&\E_{(u,y)\sim D}[(y-g(u))^2]\\
=&\E_{(u,y)\sim D}[y^2]+\E_{(u,y)\sim D}[g(u)^2]-2\E_{(u,y)\sim D}[yg(u)]\\
=&\E_{(u,y)\sim D}[y^2]+k^2\E_{(u,y)\sim D}[f(u)^2]-2k\E_{(u,y)\sim D}[yf(u)] \;.
\end{align*}
Since $y$ is supported on $\{\pm 1\}$,
we have that the first term satisfies $\E_{(u,y)\sim D}[y^2]=1 \;.$

To bound the second term, we show that $f(u)^2\leq u^2$ for any $u$.
Notice that for $u\geq t$, since $t \geq 0$ by assumption,  
we have that $f(u)^2=(u-t)^2\leq u^2$.
For $u< t$, we have that $f(u)^2=0 \leq u^2$.
Therefore, combining with the fact that $u\sim \D_{\R,1}^\gaus$, 
we can conclude that 
$$\E_{(u,y)\sim D}[g(u)^2] = k^2\E_{(u,y)\sim D}[f(u)^2]\leq k^2\E_{(u,y)\sim D}[u^2]=k^2 \;.$$
In summary, we get that 
$$\E_{(u,y)\sim D}[(y-g(u))^2]\leq 1+k^2-2k\E_{(u,y)\sim D}[yf(u)] \;.$$
We now choose the value of $k$.
If $\E_{(u,y)\sim D}[yf(u)]>0$, then we take $k=\eps$; 
otherwise, we take $k=-\eps$, in which case
we always have $k\in (-1,1)$ (since $\eps\in (0,1)$) and 
\begin{align*}
\E_{(u,y)\sim D}[(y-h(u))^2]\leq & 1+\eps^2-2\eps |\E_{(u,y)\sim D}[yf(u)]|\\
\leq & 1-\eps^2 \; . \qedhere
\end{align*}
\end{proof}

We now give a reduction similar to the proof of Theorem \ref{thm:main-thm-ltf}
using Proposition \ref{pro:lwe-short-norm}.
We know that under Assumption \ref{asm:LWE-hardness}
the following holds:
the problem
$\lwe(n^{O(k^\beta)},D_{\R^n,1}^\gaus,\mathbb{S}^{n-1},D_{\R,\sigma}^\gaus,\mod_T)$ 
with any $\sigma\geq k^{-\kappa}$ ($\kappa\in \N$ is a constant) 
and $T=1/(c'\sqrt{k\log n})$, where $c'>0$ is a sufficiently large universal constant,
cannot be solved in $n^{O(k^\beta)}$ time with $n^{-O(k^\beta)}$ advantage.
We will give an efficient reduction of the LWE problem to the problem here.

For a sample $(\bx,y')$ from a distribution $D'$ which is an instance of the problem\\
$\lwe(n^{O(k^\beta)},D_{\R^n,1}^\gaus,\mathbb{S}^{n-1},D_{\R,\sigma}^\gaus,\mod_T)$,
we will simply output $(\bx,y)$ such that: 
(i) $y=+1$ if $y'\leq T/2$, 
and (ii) $y=-1$ otherwise 
as samples from another distribution $D$.
We argue that $D$ will satisfy the following property: 
if $D'$ is from the alternative (resp. null) hypothesis of the LWE problem, 
then the resulting distribution $D$ will satisfy the alternative (resp.~null) 
hypothesis requirement of 
ReLU regression decision problem of Theorem \ref{thm:main-thm-relu}.

Since the marginal $D_{\bx}$ of $D$ satisfies $D_{\bx}=D^\gaus_{\R^n,1}$, 
it is enough to show that in the alternative hypothesis case, 
we have $R_{2}(\mathrm{ReLU}; D)=1-\Omega(1/(k\log n)^2)$, 
and in the null hypothesis case,
we have $y=+1$ with probability $1/2$ independent of $\bx$.

For the alternative hypothesis case, we first introduce the following lemma.
\begin{lemma} \label{lem: relu-correlation}
For any $\bs\in \mathbb{S}^{n-1}$, $\sigma, T\in \R_+$,  
let $D$ be the joint distribution of $(\bx,y)$ supported on $\R^n\times \{\pm 1\}$ 
such that each sample $(\bx,y)$ is generated in the following way.
We take $\bx\sim D_{\R^n,1}^\gaus, \nsv\sim D_{\R,\sigma}^\gaus$, 
and letting $y=+1$ if $\mod_T(\left<\bx, \bs\right>+\nsv)\leq T/2$ and $y=-1$ otherwise.
Given $\sigma=o(T)$, then there is a ReLU of the form $h(\bx)=\relu(\langle \bs,\bx\rangle-t)$
such that  $t\geq 0$ and 
$$\left |\E_{(\bx,y)\sim D}[yh(\bx)]\right |=\Omega(T^2)\; .$$
\end{lemma}
\begin{proof} 
We let $h_t(\bx)\eqdef\relu(\langle \bs,\bx\rangle-t)$ and $r(t)=\E_{(\bx,y)\sim D}[yh_t(\bx)]$.
Then we just need to show that 
there is a $t>0$ such that $|r(t)|=\Omega(T)$.
We observe that the derivative of $r(t)$ is
\begin{align*}
r'(t)
=&\frac{d\E_{(\bx,y)\sim D}[yh_t(\bx)]}{dt}\\
=&\frac{d\E_{(\bx,y)\sim D}[y(\langle \bs,\bx\rangle -t) \idt(\langle \bs,\bx\rangle>t)]}{dt}\\
=&-\pr_{(\bx,y)\sim D}[y=+1\land \langle \bs,\bx\rangle>t ]+\pr_{(\bx,y)\sim D}[y=-1\land \langle \bs,\bx\rangle>t ]\; ,
\end{align*}
and the second derivative of $r(t)$ is
\begin{align*}
r''(t)
=&\frac{d(-\pr_{(\bx,y)\sim D}[y=+1\land \langle \bs,\bx\rangle>t ]
+\pr_{(\bx,y)\sim D}[y=-1\land \langle \bs,\bx\rangle>t ])}{dt}\\
=&P_{\langle \bs,\bx\rangle}(t)(\pr_{(\bx,y)\sim D}[y=-1\mid \langle \bs,\bx\rangle=t]
-\pr_{(\bx,y)\sim D}[y=1\mid \langle \bs,\bx\rangle=t] )\\
=&P_{\langle \bs,\bx\rangle}(t)\left (2\pr_{(\bx,y)\sim D}[y=-1\mid \langle \bs,\bx\rangle=t]-1\right )\; .
\end{align*}
Consider the interval $t\in [T/6,T/3]$. 
Note that 
$y=-1$ 
only if $\langle \bs,\bx\rangle+z=t+z\not \in [0,T/2]$.
Thus, $y=-1$ only if $|z|\geq T/6$.
Notice $z\sim D_{\R,\sigma}^\gaus$
and $\sigma=o(T)$.
Thus, for $t\in [T/6,T/3]$, we have that 
\begin{align*}
r''(t)
=&P_{\langle \bs,\bx\rangle}(t)\left (2\pr_{(\bx,y)\sim D'}[y=-1\mid \langle \bs,\bx\rangle=t]-1\right )\\
\leq &P_{\langle \bs,\bx\rangle}(t)\left (2\pr_{z\sim D^\gaus_{\R,\sigma}}[|z|\geq T/6]-1\right )\\
=&-\Omega(1)\; ,
\end{align*}
where the last equality follows from $\sigma=o(T)$ and $P_{\langle \bs,\bx\rangle}(t)=\Omega(1)$
since $\langle \bs,\bx\rangle\sim D^\gaus_{\R,1}$ and $t\in [T/6,T/3]$ for $T<1$.

We then prove that it holds either $r(T/3)-r(T/4)=\Omega(T^2)$
or $r(T/6)-r(T/4)=\Omega(T^2)$.
First note that either $r'(T/4)\leq 0$ or $r'(T/4)> 0$.
If $r'(T/4)\leq 0$, then 
\begin{align*}
r(T/3)-r(T/4)
=&r'(T/4)(T/12)+\int_{T/4}^{T/3}r''(t)(T/3-t) dt\\
\leq &\int_{T/4}^{T/3}r''(t)(T/3-t) dt= -\Omega(T^2)\; .
\end{align*}
If $r'(T/4)>0$, then
\begin{align*}
 r(T/6)-r(T/4)
=&r'(T/4)(-T/12)+\int_{T/4}^{T/6}r''(t)(T/6-t) dt\\
\leq &\int_{T/4}^{T/6}r''(t)(T/6-t) dt=-\Omega(T^2)\; .
\end{align*}
Since either $r(T/4)-r(T/3)=\Omega(T^2)$ or $r(T/4)-r(T/6)=\Omega(T^2)$, 
then one of $|r(T/6)|, |r(T/4)|, |r(T/3)|$ must be $\Omega(T^2)$.
This completes the proof.
\end{proof}

We will apply Lemma \ref{lem: relu-correlation} on the joint distribution of $(\bx,y)$ here.
Recall that Proposition \ref{pro:lwe-short-norm} states that the LWE problem is hard for any 
fixed constant $\kappa\in \N$ and $\sigma\geq k^{-\kappa}$.
Given the constant $\gamma\in \R^+$ in this theorem, 
we will take $\kappa=\lceil 1/(2\gamma)+1/2+1\rceil$ which is a fixed constant.
Then from Proposition \ref{pro:lwe-short-norm},
the LWE problem is hard for $\sigma=k^{-\kappa}=1/(k^{3/2}\sqrt{\log n})=o(T)$.
Therefore, by Lemma \ref{lem: relu-correlation}, there is a ReLU of the form
$h(\bx)=f(\langle \bs,\bx\rangle )=\relu(\langle \bs,\bx\rangle-t)$
such that  $t\geq 0$ and 
$\left |\E_{(\bx,y)\sim D}[yh(\bx)]\right |
=\left |\E_{(\bx,y)\sim D}[yf(\langle\bx,\bs\rangle)]\right |=\Omega(T^2)=\Omega(1/(k\log n))$.
If we apply Lemma \ref{lem:correlation-implies-error} 
to the joint distribution of $(\langle\bx,\bs\rangle,y)$ and the ReLU function $f$,
we get that there must be a ReLU of the form 
$h'(\bx)=kf(\langle\bx,\bs\rangle)=\relu(\langle k\bs,\bx\rangle-kt)$ such that 
$k<1$ and 
$$\E_{(\bx,y)\sim D}[(y-h'(\bx))^2]\leq 1-\Omega(1/(k\log n)^2) \;.$$
Since $k<1$, we have that $\|k\bs\|_2\leq \|\bs\|_2= 1$,
thus $h'\in \relu$.
This implies that 
$$R_{2}(\mathrm{ReLU}; D)\leq 1-\Omega\left (1/(k\log n)^2\right ) \;.$$

For the null hypothesis case, it is immediate that $y=+1$ 
with probability $1/2$ and $y=-1$ with probability $1/2$ 
independent of $\bx$,
since $y'\sim U([0,T))$ independent of $\bx$
in the null hypothesis case of the LWE problem.
This completes the proof.
\end{proof}

The following corollary can be obtained 
directly from Theorem \ref{thm:main-thm-relu}. 

\begin{corollary} \label{cor:relu}
Under Assumption \ref{asm:LWE-hardness}, 
for any constants $\alpha\in (0,1/2)$, $\gamma>2$
and any $c/(n\log n)^2\leq \eps\leq 1/\log^{\gamma} n$
where $c$ is a sufficiently large constant, 
there is no algorithm for ReLU regression on $\R^n$ under Gaussian marginals
to error $R_{2}(\mathrm{ReLU}; D)+\eps$
and runs in time $n^{O(1/(\eps\log^2 n)^\alpha)}$.
\end{corollary}

\begin{proof}
We chose the parameter $k$ in Theorem \ref{thm:main-thm-relu} to be the value so that 
$\eps=c/(k\log n)^2$, where $c$ is a sufficiently small constant.
Then any algorithm that agnostically learns a ReLU to additive error $\eps$ 
can solve the testing problem of Theorem \ref{thm:main-thm-relu} with probability $2/3$.
Therefore, no such algorithm should run in time $n^{O(k^\beta)}$ for any $\beta\in (0,1)$.
Since $\eps=c/(k\log n)^2$, and if we chose $\beta=2\alpha$, then
the time lower bound
can be rewritten as $n^{O(k^\beta)}=n^{O(1/(\eps\log^2 n)^{\beta/2})}=n^{O(1/(\eps\log^2 n)^\alpha)}$.
This completes the proof.
\end{proof}

\bibliographystyle{alpha}
\bibliography{allrefs}

\newpage
\appendix
\section*{APPENDIX}

\section{Additional Technical Background}

For $n,k\in \N$ with $k \leq n$,
we use $S_{n,k}$ to denote the $k$-sparse set 
$S_{n,k} \eqdef \{ \bx\in \{0, \pm 1\}^n : \|\bx\|_1=k\}$.
We use $\negl(\lambda)$ to denote $\lambda^{-\omega(1)}$.

The definition of the discrete Gaussian distribution will also be useful here.
Essentially, the discrete Gaussian is a univariate discrete distribution 
supported on equally spaced points on $\R$ such that the probability mass on any point
in its support is proportional to the probability density of a Gaussian on that point. 
Following Definition \ref{def:dpart}, the discrete Gaussian distribution can be written as the following.

\begin{definition}[Discrete Gaussian] \label{def:dg}
For $T\in \R_+, y\in \R $ and $\sigma\in \R_+$, we define the 
``$T$-spaced, $y$-offset discrete Gaussian distribution with $\sigma$ scale'' 
to be the distribution of $D_{T\Z+y,\sigma}^\gaus$.
\end{definition}

Throughout our proofs, we will need to manipulate Gaussian distributions that 
are taken modulo $1$ and those with noise added to them. 
Due to this, it will be convenient to introduce the following definitions.

\begin{definition} [Expanded Gaussian Distribution from $\R_1^n$]
For $\sigma \in \R_{+}$,  let $\Dexp_{\R_1^n,\sigma}$ 
denote the distribution of $\bx'$ drawn as follows: 
first sample $\bx \sim U(\R_1^n)$, 
and then sample $\bx' \sim \Dgaus_{\Z^n+\bx,\sigma}$.
\end{definition}

\begin{definition} [Collapsed Gaussian Distribution on $\R_1^n$]
For $\sigma \in \R_{+}$, we will use $\Dcol_{\R_1^n,\sigma}$ to denote the distribution of 
$\mathrm{mod}_1(\bx)$ on $\R_1^n$, 
where $\bx\sim \Dgaus_{\R^n,\sigma}$. 
\end{definition}

\section{Hardness of cLWE with Small-Norm Secret} \label{app:clwe}

Here we give the proof of Proposition \ref{pro:lwe-short-norm}, 
which is the first step of our hardness reduction. 
Specifically, we reduce the standard discrete LWE problem in Assumption~\ref{asm:LWE-hardness}
--- where the support of $D_\smp$ is the discrete set $\Z_q^n$ --- 
into a continuous LWE (cLWE) problem
--- where the support of $D_\smp$ is  $\R^n$.
This kind of cLWE problem was first introduced in \cite{BRST21},
where the paper gives a quantum reduction
from approximating (the decision version of) 
the Shortest Vector Problem (GapSVP) to cLWE.
Subsequently, \cite{vinod2022} gave a classical reduction 
from the classic LWE problem to cLWE problem, 
indicating that cLWE problem is at least as hard as the LWE problem. 

Notably, we will not directly use the cLWE hardness statement here.
Instead, we reduce the standard discrete LWE to cLWE.
The advantage of such a reduction is that
we will be able to start from a {\em sparse} discrete LWE instance 
whose secret vector $\bs$
is sampled uniformly from $S_{n,k}$;
after the reduction, we get a cLWE instance whose dimension is $n$ 
and the $\ell_2$-norm of the secret is roughly $\sqrt{k}$ 
($\sqrt{k}\approx \log^{0.01} n$, 
compared with the $\sqrt{n}$ $\ell_2$-norm secret vector in \cite{BRST21}).

To achieve this, we slightly modify an idea from \cite{vinod2022} to get rid of 
the $\log m$ (where $m$ is the number of samples) blowup 
in the $\ell_2$-norm of the secret vector.

To prove the proposition, we start with the following lemma 
which reduces the standard LWE to an LWE 
with a $k$-sparse secret vector (i.e., a secret vector $\bs\in S_{n,k}$).

\begin{lemma} [Corollary 4 in \cite{vinod2022}] \label{lem:lwe-sparse-secret}
For any $n,m,q,l,\lambda,k\in \N$, $\sigma\in \R_+$
suppose that $\log(q)/2^l=\negl(\lambda)$,
$\sigma\geq 4\sqrt{\omega(\log \lambda)+\ln n+\ln m}$ and
$k\log_2(n/k)\geq (l+1)\log_2(q)+\omega(\log \lambda)$.
Then, if the testing problem
$\lwe(n,\Z_q^l,\Z_q^l,D_{\Z,\sigma}^\gaus,\mod_q)$ has no $T+\poly(n,m,q,\lambda)$
time distinguisher with advantage $\eps$, then the problem 
$\lwe(m,\Z_q^n,S_{n,k},D_{\Z,\sigma'}^\gaus,\mod_q)$ has
no $T$-time distinguisher with advantage $2\eps m+\negl(\lambda)$, 
where $\sigma'=2\sigma\sqrt{k+1}$. 
\end{lemma}

The above lemma reduces $\lwe(n,\Z_q^l,\Z_q^l,D_{\Z,\sigma}^\gaus,\mod_q)$
to $\lwe(m,\Z_q^n,S_{n,k},D_{\Z,\sigma'}^\gaus,\mod_q)$.
The $\lambda$ here acts as a security parameter.
Notice that the original problem $\lwe(n,\Z_q^l,\Z_q^l,D_{\Z,\sigma}^\gaus,\mod_q)$ 
has $2^{l\log q}$
possible choices of secret vector, 
while the new problem $\lwe(m,\Z_q^n,S_{n,k},D_{\Z,\sigma'}^\gaus,\mod_q)$
has roughly at least $2^{k\log_2(n/k)}$ possible choices of secret vector.
This intuitively explains why there 
is the requirement of $k\log_2(n/k)\geq (l+1)\log_2(q)+\omega(\log \lambda)$
in the lemma in terms of entropy of the secret vector.

We then use a bit of extra Gaussian noise to massage 
the noise distribution from a discrete Gaussian $\Dgaus_{\Z, \sigma}$ 
to a continuous Gaussian $\Dgaus_{\R, \sigma'}$
where $\sigma'$ is going to be slightly larger than $\sigma$. 
This leads to the following lemma:

\begin{lemma} [Lemma 15 in \cite{vinod2022}] \label{lem:lwe-continuous-noise}
Let $n,m,q,\lambda\in \N$, $\sigma\in \R_+$, $\eps\in (0,1]$
and suppose $\sigma>\sqrt{4\ln m+\omega(\log \lambda)}$.
For any $S\subseteq \R^n$,
suppose there is no $T+\poly(m,n,\log (q),\log (\sigma))$-time distinguisher 
for the problem $\lwe(m,\Z_q^n,S,D_{\Z,\sigma}^\gaus,\mod_q)$ 
with advantage $\eps$.
Then there is no $T$-time distinguisher for the problem $\lwe(m,\Z_q^n,S,D_{\R,\sigma'}^\gaus,\mod_q)$ 
with advantage $\eps+\negl(\lambda)$, where 
we set
$$\sigma'=\sqrt{\sigma^2+4\ln(m)+\omega(\log \lambda)}=O(\sigma)\; .$$
\end{lemma}

We first note that the two requirements of parameters 
in Lemma \ref{lem:lwe-continuous-noise},
$\sigma>\sqrt{4\ln m+\omega(\log \lambda)}$ and 
$\sigma'=\sqrt{\sigma^2+4\ln(m)+\omega(\log \lambda)}$
imply that $\sigma'=\sqrt{\sigma^2+4\ln(m)+\omega(\log \lambda)}=O(\sigma)$. 
This says that we are only blowing up the noise scale 
by at most a universal constant multiplicative factor.
After this lemma, we again use a bit of extra Gaussian noise to massage 
the sample distribution $D_\smp$ from $U(\Z_q^n)$ 
to $U(\R_q^n)$.
We thus obtain the following:

\begin{lemma} [Lemma 16 in \cite{vinod2022}] \label{lem:lwe-continuous-sample}
Let $n,m,q,\lambda\in \N$, $\sigma,r\in \R_+$ and $\eps\in (0,1]$. 
Let $S\subseteq\R^n$ where all elements in the support have fixed $\ell_2$-norm $r$, 
and suppose that 
$\sigma\geq 3r\sqrt{\ln n+\ln m+\omega(\log \lambda)}$.
Suppose there is no $T+\poly(m,n,\log(q),\log(\sigma))$-time distinguisher 
for $\lwe(m,\Z_q^n,S,D_{\R,\sigma}^\gaus,\mod_q)$ with advantage $\eps$,
then there is no $T$-time distinguisher for the problem $\lwe(m,\R_q^n,S,D_{\R,\sigma'}^\gaus,\mod_q)$ 
with advantage $\eps+\negl(\lambda)$, where we set 
$$\sigma'=\sqrt{\sigma^2+9r^2(\ln n+\ln m+\omega(\log \lambda))}=O(\sigma)\; .$$
\end{lemma}

Similarly, the statements $\sigma\geq 3r\sqrt{\ln n+\ln m+\omega(\log \lambda)}$ 
and $\sigma'=\sqrt{\sigma^2+9r^2(\ln n+\ln m+\omega(\log \lambda))}$ imply that 
$\sigma'=O(\sigma)$.
So to make the samples continuous, 
we are again blowing up the noise scale by at most a constant multiplicative factor. 
Then we give a modified version of Lemma 18
in \cite{vinod2022}. 
We first need to introduce the following fact from \cite{DKMR22}.

\begin{fact} [Fact A.4 in \cite{DKMR22}] \label{fct:bound-on-collaped-distribution} 
Let $n \in \N, \sigma \in \R_+, \eps \in (0, 1/3)$ be such that $\sigma \geq \sqrt{\ln(2n(1 + 1/\eps)) / \pi}$. 
Then, we have
$$\frac{P_{\Dexp_{\R_1^n,\sigma}/\sigma}(\bt)}{P_{\Dgaus_{\R^n,1}}(\bt)} 
=  \frac{P_{U(\R_1^n)}(\mathrm{mod}_1(\sigma \bt))}{P_{\Dcol_{\R_1^n,\sigma}}(\mathrm{mod}_1(\sigma \bt))} 
= 1\pm O\left (\eps\right) \;,$$
for all $\bt \in \R^n$,
and
$$d_{\mathrm{TV}}\left (\frac{\Dexp_{\R_1^n,\sigma}}{\sigma},\Dgaus_{\R^n,1}\right ), 
d_{\mathrm{TV}}\left (\Dcol_{\R_1^n,\sigma},U(\R_1^n)\right )=
\exp\left (-\Omega({\sigma^2})\right ) \;.$$
\end{fact}

Essentially, Fact \ref{fct:bound-on-collaped-distribution} says that, 
given $\bx\sim D^{\gaus}_{\R^n,\sigma}$,
the distribution of 
${\rm mod}_1 (\bx)$ is pointwise close (for its pdf function) to $U(\R_1^n)$
for sufficiently large $\sigma$. 
So if we consider the reverse of this process, given a $\bv\sim U(\R_1^n)$, 
we sample $\bu\sim D^\gaus_{\R^n+\bv,\sigma}$,
then the distribution of $\bu$ is sufficiently close to $D^{\gaus}_{\R^n,\sigma}$.
We can leverage this fact to change the sample distribution in the LWE problem
from $U(\R_q^n)$ to $D^{\gaus}_{\R^n,1}$ since $U(\R_q^n)$ is basically $U(\R_1^n)$ after rescaling.
The difference here is that the original Lemma 18 takes a large $\sigma$ so that
$d_{\mathrm{TV}}\left (\frac{\Dexp_{\R_1^n,\sigma}}{\sigma},\Dgaus_{\R^n,1}\right )\approx 1/m$,
thus $m$ samples will not see the difference.
However, since these two distributions are actually pointwise close,
we can instead take a smaller $\sigma$ 
and do an extra rejection sampling step on $\bu$ 
to make the distribution exactly a Gaussian. 
This allows us to give the nearly optimal lower bound on agnostic learning LTFs with Gaussian marginals. 
Now we give the modified version of Lemma 18 in \cite{vinod2022}.

\begin{lemma} [Modified Lemma 18 in \cite{vinod2022}] \label{lem:lwe-gaussian-sample}
Let $n,m,q\in \N,\sigma,r,\alpha\in \R_+$.
Let $S\subseteq\Z^n$ where all elements in the support have fixed $\ell_2$-norm $r$.
Suppose there is no $T+\poly(n,m,\log(q))$-time distinguisher for
the problem $\lwe(m,\R_q^n,S,D_{\R,\sigma}^\gaus,\mod_q)$ with $\eps$ advantage. 
Then there is no $T$-time distinguisher for 
the problem $\lwe(m',D_{\R^n,1}^\gaus, S/r ,D_{\R, \alpha\sigma/q}^\gaus,\mod_\alpha)$ with 
$\eps+2^{-\Omega(m)}$ advantage, where
$$\alpha=c/\left (r\sqrt{\log n}\right )\; ,$$
$$m'=cm\; ,$$
and $c>0$ is a sufficiently small universal constant.
\end{lemma}

\begin{proof}
We will give a reduction argument.
Given a sample $(\bx,y)$ from 
$\lwe(m,\R_q^n,S,D_{\R,\sigma}^\gaus,\mod_q)$,
we can generate a sample $(\bx',y')$ from
the problem $\lwe(m',D_{\R^n,1}^\gaus, S/r ,D_{\R,\alpha\sigma/q}^\gaus,\mod_\alpha)$ 
with at least a constant success probability in the following manner.

We take a $\tilde\sigma=1/r\alpha$ and sample $\tbx\sim D^\gaus_{\Z^n+\bx/q,\tilde\sigma}/\tilde\sigma$.
We define the function $f:\R^n\to \R$ as 
$$f(\bt)\eqdef\frac{P_{\Dgaus_{\R^n,1}}(\bt)}{P_{\Dexp_{\R_1^n,\tilde\sigma}/\tilde\sigma}(\bt)} \;.$$
With probability $f(\tbx)/\max_{\bt\in \R^n}f(\bt)$, 
we take $\bx'=\tbx$ and $y'=y/\left (qr\tilde\sigma\right )$ and
output $(\bx', y')$ as a sample for 
$\lwe(m',D_{\R^n,1}^\gaus, S/r ,D_{\R,\alpha\sigma/q}^\gaus,\mod_\alpha)$.
Otherwise, we output failure.

We will prove that if $(\bx,y)$ is from the alternative hypothesis case, then it must be 
$\bx'\sim D^{\gaus}_{\R^n,1}$ and $y'=\mod_\alpha(\langle \bs',\bx'\rangle +z')$, 
where $\bs'\sim U(S/q)$ and $z'\sim D^\gaus_{\R,\alpha\sigma/q}$.
Since $(\bx,y)$ is from the alternative hypothesis case, it must satisfy 
$\bx\sim U(\R_q^n)$ and $y=\mod_q(\langle \bs,\bx\rangle +z)$,
where $\bs\sim U(S)$ and $z\sim D^\gaus_{\R,\sigma}$.
Then, the fact $\tbx\sim D^\gaus_{\Z^n+\bx/q,\tilde\sigma}/\tilde\sigma$
implies that $\tilde\sigma\tbx-\bx/q\in \Z^n$ 
and $q\tilde\sigma\tbx-\bx\in q\Z^n$; combined with $\bs\in \Z^n$, 
we have that 
$$\mod_q(\langle \bs,\bx\rangle)
=\mod_q(\langle \bs,q\tilde\sigma\tbx\rangle+\langle \bs,\bx-q\tilde\sigma\tbx\rangle)
=\mod_q(\langle \bs,q\tilde\sigma\tbx\rangle)\; .$$
Then we can write
\begin{align*}
y'
=&y/\left(qr\tilde\sigma\right)\\
=&\mod_q(\langle \bs,\bx\rangle +z)/\left(qr\tilde\sigma\right)\\
=&\mod_q(\langle \bs,q\tilde\sigma\tbx\rangle +z)/\left(qr\tilde\sigma\right)\\
=&\mod_1(\langle \bs,\tilde\sigma\tbx\rangle+z/q)/\left(r\tilde\sigma\right)\\
=&\mod_{1/\left (r\tilde\sigma\right )}
\left (\left \langle \bs/r,\tbx\right \rangle+z/\left (qr\tilde\sigma\right )\right )\\
=&\mod_\alpha
\left (\left \langle \bs/r,\bx'\right \rangle+\alpha z/q\right )\;,
\end{align*}
where the last equality follows from the fact 
$\tilde\sigma=1/(r\alpha)$.
Note that the three terms in the above expression, $\bs/r$, $\bx'$ and $\alpha z/q$ 
are independent (since $\bx',\bs,z$ are independent).
It only remains to verify the distribution of each of them.

It is immediate that $\bs/r\sim U(S/r)$.
For the other two, we first define the following notation.
For functions $f, g: U \to \R$, we write $f(u)\propto g(u)$ 
if there is a constant $c\in \R \setminus \{0\}$
such that for all $u\in U$, it holds $f(u)=cg(u)$.
For $\bx'$, we first notice that $\bx/q\sim \R_1^n$, 
and therefore $\tbx\sim \Dexp_{\R_1^n,\tilde\sigma}/\tilde\sigma$.
Combining with the rejection sampling procedure we performed,
we have that 
$$P_{\bx'}(\bu)\propto\frac{f(\bu)}{\max_{\bt\in \R^n}f(\bt)}P_{\tbx}(\bu)
=\frac{f(\bu)}{\max_{\bt\in \R^n}f(\bt)}P_{\Dexp_{\R_1^n,\tilde\sigma}/\tilde\sigma}(\bu)
=\frac{P_{\Dgaus_{\R^n,1}}(\bt)}{\max_{\bt\in \R^n}f(\bt)}
\propto P_{\Dgaus_{\R^n,1}}(\bu) \; .$$
Thus, we conclude that $\bx'\sim \Dgaus_{\R^n,1}$.
For $\alpha z/q$, 
notice that $z\sim D^\gaus_{\R,\sigma}$,
and therefore $\alpha z/q\sim D^\gaus_{\R,\alpha\sigma/q}$.

For the null hypothesis case, it is easy to see that 
the marginals satisfy
$D_{\bx'}=D^\gaus_{\R^n,1}$ and $D_{y'}=U(\R_\alpha)$,
and $\bx'$ and $y'$ are independent --- since
$\bx$ and $y$ are independent and $\bx'$ (resp. $y'$) only depends on $\bx$ (resp $y$).

It remains to verify that the sampling will produce at least $m'$ 
many samples with $1-2^{-\Omega(m)}$ probability.
We first show that each individual rejection sampling 
succeeds with at least a positive constant probability.
From Fact \ref{fct:bound-on-collaped-distribution}, 
given $\tilde \sigma =1/r\alpha=\sqrt{\log n}/c$ 
for sufficiently small constant $c>0$, we have
$$f(\bt)=\frac{P_{\Dgaus_{\R^n,1}}(\bt)}{P_{\Dexp_{\R_1^n,\tilde\sigma}/\tilde\sigma}(\bt)} \in (1/2,3/2) \;.$$
Notice that for any $\bx$, we accept the sample with $f(\tbx)/\max_{\bt\in \R^n}f(\bt)$ probability,
which is at least $1/3$ probability given the bound above. 
Then, by an application of the Chernoff bound, 
we have that the rejection sampling succeeds at least $m'=cm$ times
with probability at least $1-2^{-\Omega(m)}$, where 
$c>0$ is a sufficiently small constant. 
This completes the proof.
\end{proof}

We note that Lemma \ref{lem:lwe-gaussian-sample} is stronger 
than Lemma 18 in \cite{vinod2022} in the sense
that the original Lemma 18 has 
$\alpha=c/\left (r\sqrt{\log n+\log m+\omega(\log \lambda)}\right )$, 
compared with $\alpha=c/\left (r\sqrt{\log n}\right )$ here.
For the task of learning LTFs, 
if one uses Lemma 18 instead of Lemma \ref{lem:lwe-gaussian-sample} 
and follows the same argument for rest of the proof,
one will still get am $n^{\Omega(1/(\eps\sqrt{\log n})^{0.99})}$ lower bound 
--- compared with the $n^{\Omega(1/(\eps\sqrt{\log n})^{1.99})}$ 
\emph{near-optimal} lower bound we establish here.

Combining the above lemmas and Assumption \ref{asm:LWE-hardness}, 
we establish the proof of Proposition \ref{pro:lwe-short-norm}.

\begin{proof} [Proof of Proposition \ref{pro:lwe-short-norm}]
We provide an efficient reduction from Assumption \ref{asm:LWE-hardness} via
Lemma \ref{lem:lwe-sparse-secret}, 
Lemma \ref{lem:lwe-continuous-noise}, 
Lemma \ref{lem:lwe-continuous-sample} and 
Lemma \ref{lem:lwe-gaussian-sample}.
More precisely, the reduction will follow the following steps:
\begin{enumerate}
\item Let the problem in Assumption \ref{asm:LWE-hardness}
be solving $\lwe(2^{O(l^{\beta'})}, \Z_q^l,\Z_q^l, D^\gaus_{\Z,\sigma'},\mod_q)$ 
with $2^{-O(l^{\beta'})}$ advantage, where $l$ is the dimension.

\item We then use Lemma \ref{lem:lwe-sparse-secret} to reduce to  
solving the problem $\lwe(n^{O(k^{\beta})},\Z_q^n, S_{n,k}, D^\gaus_{\Z,c\sqrt{k}\sigma'},\mod_q)$ 
with $n^{-O(k^{\beta})}$ advantage, 
where 
$c$ is a sufficient large positive universal constant,
$n$ is the dimension
and the secret vector is from the sparse set $S_{n,k}$.

\item The we apply Lemma \ref{lem:lwe-continuous-noise}
and Lemma \ref{lem:lwe-continuous-sample}.
The two lemmas make the sample and noise distributions continuous.
As we argued before, these two lemmas will only blow up the noise scale by a universal constant factor, 
so we reduce to
solving 
$\lwe(n^{O(k^{\beta})},\R_q^n, S_{n,k}, D^\gaus_{\R,c\sqrt{k}\sigma'},\mod_q)$ 
with $n^{-O(k^{\beta})}$ advantage, 
where $c$ is a sufficiently large positive universal constant.

\item To finish the reduction, we apply Lemma \ref{lem:lwe-gaussian-sample} 
which mainly changes the sample distribution from $U(\R_q^n)$ to $D^\gaus_{\R^n,1}$
and reduce to 
solving the problem
$\lwe(n^{O(k^{\beta})},D^\gaus_{\R^n,1}, \mathbb{S}^{n-1}, D^\gaus_{\R,\sigma},\mod_\alpha)$
with $n^{-O(k^{\beta})}$ advantage.
\end{enumerate}

To start the reduction, we need to chose the values for parameters $l,\beta',q,\sigma'$ in the first step.
Let $n,k,\beta,\gamma, \kappa$ be the parameters in the body of Proposition \ref{pro:lwe-short-norm} 
which are the target parameters we want to get after the reduction.
For convenience, we let $\delta>0$ be the constant such that $1-3\delta=\beta$.
Let $\psi$ be the value such that $k=\log^{\psi} n$ ($\psi$ has dependence on $n$ and $k$).
We will chose the following values:
\begin{itemize}
	\item $l=\log^{t} n$, where $t=1+\psi(1-\delta)$;
	\item $\beta'=\frac{1+\gamma (1-2\delta)}{1+\gamma(1-\delta)}$, which is a constant, and $\beta'\in (0,1)$;
	\item $q=k^{\kappa+1}$;
	\item $\sigma'=c\sqrt{l}$, where $c$ is a sufficiently large constant.	
\end{itemize}
We now check validity of the parameters for each step of the reduction:
\begin{enumerate}[leftmargin=*]
\item We first check that the parameters satisfy the requirements in Assumption \ref{asm:LWE-hardness}.
Notice that 
$$q=k^{\kappa+1}=\log^{\psi(\kappa+1)} n=l^{\psi(\kappa+1)/t}\leq l^{\psi(\kappa+1)/(\psi(1-\delta))}
=l^{(\kappa+1)/(1-\delta)}=l^{O(1)}\; .$$

\item We then check the requirements in Lemma \ref{lem:lwe-sparse-secret}. 
We chose the additional parameters as $\lambda=2^{l^{\beta'}}$ and $m=n^{O(k^\beta)}$.
For convenience, we first show that $2^{l^{\beta'}}=n^{\omega(k^\beta)}$.
Notice that
$$2^{l^{\beta'}}=2^{\log^{t\beta'} n}=n^{\log^{t\beta'-1} n}=n^{k^{\frac{t\beta'-1}{\psi}} }\; .$$
Since $k=\log^\psi n$ and $k\geq \log^\gamma n$, it follows that $\psi\geq \gamma$; 
therefore, $\beta'=\frac{1+\gamma (1-2\delta)}{1+\gamma(1-\delta)}\geq \frac{1+\psi (1-2\delta)}{1+\psi(1-\delta)}$.
Plugging this into the above, we get that 
$$2^{l^{\beta'}}\geq n^{k^{\frac{t\frac{1+\psi (1-2\delta)}{1+\psi(1-\delta)}-1}{\psi}} }= n^{k^{1-2\delta}}=n^{\omega(k^\beta)}\; ,$$
where the last equality follows from the fact $\beta=1-3\delta$.
For the requirements, we have: 
\begin{enumerate}[leftmargin=*]
	\item It is immediate that $\log(q)/2^l=O(\log l)/2^l=\negl(\lambda)$ (since $q=l^{O(1)}$ from the last step).

	\item \label{item:step-2-parameter} For the requirement $\sigma'\geq 4\sqrt{\omega(\log \lambda)+\ln n+\ln m}$, 
	since $\sigma'=c\sqrt{l}$, taking squares on both side, it can be rewritten as 		
	$$l=\omega(\log \lambda+\ln n+\ln m) \;.$$ 
	Notice that $\log \lambda =O(l^{\beta'})$, where $\beta'<1$; thus, $l= \omega( \log \lambda)$. 
	Since $l=\log^t n$, where $t=1+\psi(1-\delta)\geq 1+\gamma(1-\delta)$ and $\gamma(1-\delta)$ is a positive constant, 
	we have that $l=\omega(\ln n)$.
	Then, since $2^{l^{\beta'}}=n^{\omega(k^\beta)}$ as shown above, and $m=n^{O(k^\beta)}$,
	we get that $2^{l^{\beta'}}=\omega(m)$; thus, 
	we get $l=\omega(l^{\beta'})=\omega(\log m)$.
	Combining the above gives us that $l=\omega(\log \lambda+\ln n+\ln m)$. 

	\item For the requirement $k\log_2(n/k)\geq (l+1)\log_2(q)+\omega(\log \lambda)$, 
	since $l=\omega(\log \lambda)$ as shown above, 	
	we can rewrite it as $k\log_2(n)-k\log_2(k)\geq 2l\log_2(q)$.
	Since $q=\poly(l)$ from step 1, it therefore suffices to show that
	$k\log n-k\log k= \omega(l\log l)$,
	which is $k\log n\geq cl\log l +k\log k$ for any constant $c$.
	We prove this by analyzing two cases, 
	namely $cl\log l\leq k\log k$ and $cl\log l> k\log k$.
	
	If $cl\log l\leq k\log k$, then since $k<c'n$, where $c'$ is a sufficiently small universal constant,
	we get that $k\log n\geq 2k\log k \geq cl\log l +k\log k$.
	
	If $cl\log l> k\log k$, then it suffices to show that $k\log n=\omega(l\log l)$.
	Notice that 
	$k\log n=\log^{1+\psi} n$
 	and $l\log l=t\log^t n\log\log n $.
	Thus, 
	$$\frac{k\log n}{l \log l}=\frac{\log^{1+\psi-t} n}{t\log\log n}\; .$$
	Notice that $1+\psi-t=\delta\psi\geq \delta\gamma$ 
	(since $k\geq \log^\gamma n$ and $k=\log^\psi n$ implies $\psi\geq\gamma$) 
	is at least a constant; thus, 
	$$\frac{k\log n}{l \log l}
	=\frac{\log^{1+\psi-t} n}{t\log\log n}
	=\frac{\log^{\delta\psi} n}{t\log\log n}
	=\omega\left (\frac{\log^{\delta\psi/2} n}{t}\right )
	=\omega\left (\frac{\log^{\delta\psi/2} n}{1+\psi}\right )\; ,$$
	where the last equality comes from the fact that $t=1+\psi(1-\delta)\leq 1+\psi$.
	Therefore, we just need to show that 
	$\frac{\log^{\delta\psi/2} n}{1+\psi}$ is at least a constant.
	Notice that for any sufficiently large $n$ such that $\log^{\delta/2} n\geq e$, we have that
	$$\log^{\delta\psi/2} n=(\log^{\delta/2} n)^\psi \geq e^\psi\geq 1+\psi \;.$$
	Thus, we have that 
	$$\frac{k\log n}{l \log l}=\omega(1)\; ,$$ 
	which is $k\log n=\omega(l\log l)$.
	
	Therefore, the requirement $k\log_2(n/k)\geq (l+1)\log_2(q)+\omega(\log \lambda)$ 
	is satisfied in both cases.

	\item It only remains to verify the 
	time lower bound of $2^{-O(l^{\beta'})}$ and  
	advantage $2\eps m+\negl(\lambda)$ in Lemma \ref{lem:lwe-sparse-secret}, 
	where $\eps$ is the advantage before the reduction.	
	Notice that since $2^{l^{\beta'}}=n^{\omega(k^\beta)}$, 
	the time lower bound is at least any $n^{O(k^\beta)}$. 
	For the advantage, by taking $\eps=2^{-3(l^{\beta'})}$, we have that
	$$2\eps m+\negl(\lambda)=2^{-3(l^{\beta'})}n^{O(k^\beta)}+\negl(2^{l^{\beta'}})
	\leq 2^{-2(l^{\beta'})}+\negl(2^{l^{\beta'}})= n^{-\omega(k^\beta)} \;,$$
	where the last inequality and equality follows from the statement 
	$2^{l^{\beta'}}=n^{\omega(k^\beta)}$ shown above.
	Thus, there is no $n^{O(k^\beta)}$-time distinguisher for solving 
	$\lwe(n^{O(k^{\beta})},\Z_q^n, S_{n,k}, D^\gaus_{\Z,c\sqrt{k}\sigma'},\mod_q)$ 
	with $n^{-O(k^{\beta})}$ advantage. 
	\end{enumerate}

\item We then check the parameter requirements in Lemma \ref{lem:lwe-continuous-noise}
and Lemma \ref{lem:lwe-continuous-sample}. 
Note that it suffices to check that 
$c\sqrt{k}\sigma'\geq 3r\sqrt{\ln n+\ln m+\omega(\log \lambda)}$ for sufficiently large constant $c$.
Since $r=\sqrt{k}$ from its definition 
and we have already shown that $\sigma'\geq 4\sqrt{\omega(\log \lambda)+\ln n+\ln m}$
in Step \ref{item:step-2-parameter}, this inequality holds.

Then it only remains to verify the time lower bound and advantage.
The time lower bound is $n^{ck^{\beta}}-\poly(m,n,\log(q),\log(c\sqrt{k}\sigma'))$.
Since $m=n^{O(k^{\beta})}$, 
$\log (q)=\log (k^{\kappa+1})=O(\log k)$,
and $\log(c\sqrt{k}\sigma')=O(\log k+\log l)= O(\log^{1+\psi} n)=O(k\log n)$,
by choosing $c$ to be a sufficiently large constant, 
the above lower bound is any $n^{O(k^{\beta})}$.
Similarly, the advantage is any $n^{-O(k^{\beta})}$.
Thus, there is no $n^{O(k^\beta)}$-time distinguisher for solving the problem 
$\lwe(n^{O(k^{\beta})},\R_q^n, S_{n,k}, D^\gaus_{\R,c\sqrt{k}\sigma'},\mod_q)$ 
with $n^{-O(k^{\beta})}$ advantage.

\item  After applying Lemma \ref{lem:lwe-gaussian-sample}, we get that 
there is no $n^{O(k^\beta)}$-time distinguisher for solving the problem 
$\lwe(m',D^\gaus_{\R^n,1}, S_{n,k}/\sqrt{k}, D^\gaus_{\R,c\alpha\sqrt{k}\sigma'/q},\mod_\alpha)$
with $n^{-O(k^{\beta})}$ advantage,
where 
$\alpha=c/\left (\sqrt{k\log n}\right)$, $m'=cn^{O(k^\beta)}$,
and $c>0$ is a sufficiently small universal constant.
We just need to check that it matches the values of $\sigma, m, T$ 
in the body of Proposition \ref{pro:lwe-short-norm}.
For the noise scale $\sigma$, we have
$$c\alpha\sqrt{k}\sigma'/q=c'\sqrt{l}/(\sqrt{\log n} q) =c'\log n^{\psi(1-\delta)/2}/q \leq c'k^{1/2}/k^{\kappa+1}=o( k^{-\kappa})=o( \sigma)\; ,$$
where the last inequality follows from $k=\log^\psi n$.
For the number of samples, we have that
$m'=cn^{c'(k^\beta)}$ which is any $n^{O(k^\beta)}$ 
by choosing $c'$ to be sufficiently large.
For the parameter $T$,
we have that $\alpha=c/\left (\sqrt{k\log n}\right)=T$.
Then, the only remaining difference is that the secret vector 
distribution is $U(S_{n,k}/\sqrt{k})$ instead of 
$U(\mathbb{S}^{n-1})$. 
The catch here is that we can do a random rotation 
on all the samples and this makes the secret vector also randomly rotated
and gives the $U(\mathbb{S}^{n-1})$ distribution we want.
Therefore, there is no $n^{O(k^\beta)}$-time distinguisher for solving the problem 
$\lwe(n^{O(k^\beta)},D^\gaus_{\R^n,1}, \mathbb{S}^{n-1}, D^\gaus_{\R,\sigma},\mod_T)$
with $n^{-O(k^{\beta})}$ advantage.
\end{enumerate}
This proves Proposition \ref{pro:lwe-short-norm}.

\end{proof}

\end{document}